%% file: main.tex
\documentclass{article}

\usepackage{microtype}
\usepackage{graphicx}
\usepackage{subfigure}
\usepackage{booktabs} % for professional tables

\usepackage{hyperref}

\usepackage[accepted]{icml2019}

\usepackage[utf8]{inputenc} % allow utf-8 input
\usepackage[T1]{fontenc}    % use 8-bit T1 fonts
\usepackage{hyperref}       % hyperlinks
\usepackage{url}            % simple URL typesetting
\usepackage{booktabs}       % professional-quality tables
\usepackage{amsfonts}       % blackboard math symbols
\usepackage{nicefrac}       % compact symbols for 1/2, etc.
\usepackage{microtype}      % microtypography

\usepackage{algorithm}
\usepackage{algorithmic}
\usepackage{amsmath}
\usepackage{amssymb}
\usepackage{amsthm}
\usepackage{subfigure}
\usepackage{wrapfig}
\usepackage{xspace}
\usepackage{enumitem}

\usepackage{tikz}
\usetikzlibrary{shapes,arrows,shadows,fit, decorations.pathmorphing}
\tikzstyle{op} = [draw, fill=red!20, text centered,
                  minimum height=2em, circle]
\tikzstyle{mlp} = [draw, text width=5em, fill=blue!20, text centered,
                  minimum height=2em, rounded corners]
\tikzstyle{inp} = [text centered, minimum height=2em]
\tikzstyle{system} = [draw, dotted, minimum height=2em]
\tikzstyle{null} = [inner sep=0, outer sep=0]
\tikzset{>=latex}
\tikzset{every picture/.style={line width=1pt}}
\newtheorem{assume}{Assumption}

\newtheorem{lemma}{Lemma}
\newtheorem{theorem}{Theorem}

\newcommand{\DF}{DF\xspace}
\newcommand{\FF}{FF\xspace}
\newcommand{\RFF}{RFF\xspace}
\newcommand{\base}{{\rm Base}\xspace}
\newcommand{\regret}{{\rm Regret}\xspace}

\newcommand{\X}{\mathcal{X}}
\newcommand{\Y}{\mathcal{Y}}
\newcommand{\Z}{\mathcal{Z}}

\newcommand{\A}{\mathcal{A}}
\newcommand{\E}{\mathbb{E}}

\usepackage[normalem]{ulem}

\usepackage[compact]{titlesec}
\titlespacing{\section}{0pt}{2ex}{1ex}
\titlespacing{\subsection}{0pt}{1ex}{0ex}
\titlespacing{\subsubsection}{0pt}{0.5ex}{0ex}

\icmltitlerunning{Learning from Delayed Outcomes via Proxies with Applications to Recommender Systems}

\begin{document}
% John Quan's comment on restoring page numbers for ICML format
% For posterity (or if you're working on an arXiv submission), it's the \fancyhf{} command in the ICML style file that removes the page numbers. One way to get them back is to add \cfoot{\thepage} after \begin{document} and have something like \setlength{\footskip}{3em} to position the page numbers better. Now all pages but the first will have a page number. To fix the first page, change the line with \thispagestyle{empty} to \thispagestyle{plain} in the ICML style file.
%\cfoot{\thepage}
%\setlength{\footskip}{3em}

\twocolumn[
\icmltitle{Learning from Delayed Outcomes via Proxies with Applications to Recommender Systems}

% It is OKAY to include author information, even for blind
% submissions: the style file will automatically remove it for you
% unless you've provided the [accepted] option to the icml2018
% package.

% List of affiliations: The first argument should be a (short)
% identifier you will use later to specify author affiliations
% Academic affiliations should list Department, University, City, Region, Country
% Industry affiliations should list Company, City, Region, Country

% You can specify symbols, otherwise they are numbered in order.
% Ideally, you should not use this facility. Affiliations will be numbered
% in order of appearance and this is the preferred way.
\icmlsetsymbol{equal}{*}

\begin{icmlauthorlist}
\icmlauthor{Timothy A. Mann}{equal,dm}
\icmlauthor{Sven Gowal}{equal,dm}
\icmlauthor{Andr\'as Gy\"orgy}{dm}
\icmlauthor{Ray Jiang}{dm}
\icmlauthor{Huiyi Hu}{dm}
\icmlauthor{Balaji Lakshminarayanan}{dm}
\icmlauthor{Prav Srinivasan}{dm}
\end{icmlauthorlist}
\icmlaffiliation{dm}{DeepMind, London, UK}
\icmlcorrespondingauthor{Timothy A. Mann}{timothymann@google.com}

% You may provide any keywords that you
% find helpful for describing your paper; these are used to populate
% the "keywords" metadata in the PDF but will not be shown in the document
\icmlkeywords{Recommender systems, delayed feedback, intermediate feedback}

\vskip 0.3in
]

% this must go after the closing bracket ] following \twocolumn[ ...

% This command actually creates the footnote in the first column
% listing the affiliations and the copyright notice.
% The command takes one argument, which is text to display at the start of the footnote.
% The \icmlEqualContribution command is standard text for equal contribution.
% Remove it (just {}) if you do not need this facility.

%\printAffiliationsAndNotice{}  % leave blank if no need to mention equal contribution
\printAffiliationsAndNotice{\icmlEqualContribution} % otherwise use the standard text.

\begin{abstract}
Predicting delayed outcomes is an important problem in recommender systems (e.g., if customers will finish reading an ebook). We formalize the problem as an adversarial, delayed online learning problem and consider how a proxy for the delayed outcome (e.g., if customers read a third of the book in 24 hours) can help minimize regret, even though the proxy is not available when making a prediction. Motivated by our regret analysis, we propose two neural network architectures: Factored Forecaster (\FF) which is ideal if the proxy is informative of the outcome in hindsight, and Residual Factored Forecaster (\RFF)  that is robust to a non-informative proxy. Experiments on two real-world datasets for predicting human behavior show that \RFF outperforms both \FF and a direct forecaster that does not make use of the proxy. Our results suggest that exploiting proxies by factorization is a promising way to mitigate the impact of long delays in human-behavior prediction tasks.
\end{abstract}

\section{Introduction}

Predicting delayed outcomes is an important problem in recommender systems and online advertising. The problem is aggravated by a stream of new items for recommendation while successful recommendations are determined by a delayed outcome (e.g., on the order of weeks or months). The actual time that an outcome is delayed is less important than the number of times a forecaster is queried for a new item without receiving any feedback. For example, in a popular recommender system an outcome delayed by a week might mean that the forecaster makes millions of predictions for a popular new item without incorporating any feedback into the forecaster's model.

{\bf Motivating Example:} Consider an online marketplace for ebooks. Our problem is to recommend ebooks that a customer will finish within 90 days of its purchase. Similarly to YouTube video recommendations \citep{Davidson2010}, when a customer visits, the marketplace (1) generates a small set of ebook candidates, (2) scores each of the candidates, and (3) presents them in descending order of predicted probability of finishing the book. When an ebook is purchased, the outcome is unknown for 90 days. The approach taken by much of the prior work on learning with delay \citep{Weinberger2002,Mesterharm2005,Joulani2013,Quanrud2015} would require waiting the whole 90 days to see the outcome. However, new books are added to the marketplace every day.
% TODO(kingtim@) Make this statement more precise and provide citations.
If we waited 90 days before we can accurately predict the engagement probability of a new book, the marketplace will miss out on many potential sales. While generalization can help to mitigate this problem, per product memorization is important for large scale recommender systems \citep{Cheng2016}. A complementary approach could make use of a less-delayed proxy for the delayed outcome. 
For example, one day after a purchase we can define a proxy based on the furthest page reached in the ebook. Clearly, this proxy provides some information about the eventual outcome. On the other hand, we are interested in predicting the eventual outcome, not the proxy. This leads to the main question studied in this paper: How can we learn to predict delayed outcomes via a proxy?

% How do we formalize the problem?
{\bf The Problem:} We formulate learning from delayed outcomes as a full information online learning problem where the goal is to minimize regret \cite{Hazan2016}. The online learning framework is ideal for predicting human responses (e.g., recommender systems), because it does not make distributional assumptions about the data. In our motivating example, ebook candidates are not generated by a fixed distribution, since new books are being added and the popularity of books changes over time. For these reasons, we model the problem as a delayed, adversarial online learning problem with less-delayed proxies, where instances are selected by an oblivious adversary but the proxies and outcomes are sampled in a probabilistic manner conditioned on the instances. Furthermore, we assume that the number of proxy symbols and the number of outcomes are small relative to the number of instances (e.g., the number of ebooks in an online marketplace). 

{\bf Proposed Solution:} We suggest that a proxy for the delayed outcome can be exploited by breaking the prediction problem into two subproblems: (a) predicting the proxy from instances, and (b) predicting the delayed outcome from the proxy. The intuition is that the first model can be updated quickly, since the proxy is less delayed than the outcome. On the other hand, the model predicting the outcome can only be updated when outcomes are revealed but it only uses the proxy as input. Thus, it can generalize across instances. We analyze the regret of this factored approach and find that it scales with $(D + N)$ where $D$ is the outcome delay and $N$ is the total number of items. On the other hand, the regret of forecasters that ignore the proxy may scale with $(D \times N)$. Thus, proxy information can mitigate the impact of delay.

{\bf Practical Neural Implementations:} Using this intuition we propose a factored neural network-based online learner, called factored forecaster (\FF), based on two modules -- one predicts the proxies, while the other predicts the delayed outcomes. However, we found that the selected proxies are not always sufficiently informative of the outcomes in both of our experimental domains. To mitigate this problem, we introduce a second neural network-based learner, called ``factored with residual'' forecaster (\RFF), that introduces a residual correction term. We compare both of these factored algorithms to an algorithm that ignores the proxies, called the direct forecaster (\DF), on two experimental domains: (1) predicting commit activity for GitHub repositories, and (2) predicting engagement with items acquired from a popular marketplace. In both of these domains, \RFF\ outperforms both \DF\ and \FF.

% Contributions
\noindent {\bf Contribution:} This paper offers three main contributions:
\begin{enumerate}[wide=0pt,leftmargin=0pt]
    \item We formalize the problem of learning from delayed outcomes via proxies as a full-information online learning problem. Many prior works have analyzed learning with delayed outcomes \citep[e.g.,][]{Weinberger2002,Mesterharm2005,Joulani2013}; however, to our knowledge, this is the first work to consider using proxies to mitigate the impact of delay.
    \item We analyze the regret of learning with proxies. In particular, factorization helps when the forecaster is forced to make a burst of predictions for the same instance (i.e., predictions for a popular new item are requested many times before any outcome is revealed).
    \item Finally, we introduce a practical neural network implementation, \RFF, for exploiting proxies. Our experiments provide evidence that \RFF performs as well as a delayed learner when the proxies are not informative of the outcomes and outperforms a delayed learner when instances are chosen adversarially and the proxies are informative of the outcomes.
\end{enumerate}

\section{Formal Problem Description \& Approach} 
\label{sec:formal}

% How does the formal setting relate back to recommender systems?
The formal setting in this paper is motivated by online marketplaces where we want to predict consumer engagement with purchased products. An instance $x \in \X$ represents features about a user and a content being considered for recommendation, while a label $y \in \Y$ indicates what kind of engagement occurred, with the simplest being $\Y = \{0, 1\}$ where $y = 0$ represents ``no engagement'' and $y = 1$ represents ``successful engagement''. A proxy $z \in \Z$ represents an observation made after the prediction, but before the outcome is revealed. 
Typically, the number of proxy symbols will be much smaller than the number of instances.

Above, $\X$, $\Y$, $\Z$ are nonempty finite sets representing the sets of potential instances, outcomes, and proxies, respectively.
 We denote the number of potential instances $|\X|$ by $N$, and the $(d-1)$-dimensional simplex by $\Delta^d$ for any $d \geq 2$. Let $T > 0$ denote the number of prediction rounds, and for simplicity we assume that every outcome is observed with a fixed delay $D \ge 0$ ($D=0$ implies no delay). Before an episode begins, the environment generates $(x_t, y_t, z_t) \in \X \times \Y \times \Z$ for each round $t \in \{1, 2, \dots, T\}$.
\begin{assume} \label{asm:almost_adversarial}
Let $g(\cdot|x)$ be a conditional probability distribution over $\Z$ for all $x \in \X$, and $h(\cdot|z)$ be a conditional probability distribution over $\Y$ for all $z \in \Z$. The sequence $\langle x_t \rangle_{t=1}^T$ is generated arbitrarily. The sequence $\langle z_t \rangle_{t=1}^T$ is generated randomly such that $z_t \sim h(\cdot| x_t)$ and 
%is conditionally independent of $\langle z_1,\ldots,z_{t-1},z_{t+1},\ldots,z_T\rangle$ given $x_t$.
the random variables $\langle z_t \rangle_{t=1}^T$ are conditionally independent given $\langle x_t \rangle_{t=1}^T$.
Furthermore, the sequence $\langle y_t \rangle_{t=1}^T$ is generated randomly such that $y_t \sim g(\cdot| z_t)$ and
%it is conditionally independent of $\langle y_1,\ldots,y_{t-1},y_{t+1},\ldots,y_T\rangle$ given $z_t$.
the random variables $\langle y_t \rangle_{t=1}^T$ are conditionally independent given $\langle z_t \rangle_{t=1}^T$.
\end{assume}
That is, we make no assumption about how the instances are selected, while the proxies and outcomes are generated stochastically from the unknown distributions $h$ and $g$.

\begin{figure}
    \centering
    \includegraphics[width=0.45\textwidth]{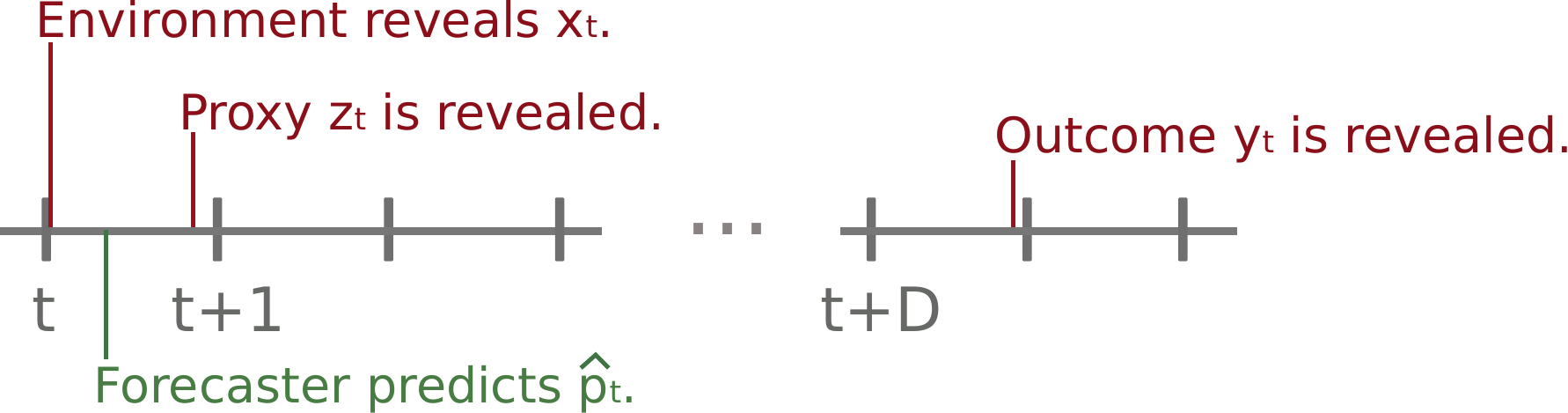}
    \caption{A timeline depicting prediction of a delayed outcome. At the beginning of round $t$, the environment reveals the instance $x_t$. Next the forecaster makes a prediction $\hat{p}_t$ given $x_t$. At the end of round $t$, the proxy $z_t$ is revealed only after the forecaster makes a prediction. Finally, the outcome $y_t$ is revealed at the end of round $t+D$ (after prediction $\hat{p}_{t+D}$ is made).}
    \label{fig:delayed_prediction_timeline}
\end{figure}

In our theoretical analysis, for clarity, we consider a setup where each outcome is observed with a fixed delay $D$, while the proxies are revealed with no delay, as described by the following model:%
\footnote{We consider delayed proxies in our experiments.}
At each round $t \geq 1$, using the instance $x_t$, the forecaster makes a prediction $\hat{p}_t \in \mathcal{P}=\{ \phi \mid \phi : \X \rightarrow \Delta^{|\Y|}\}$ and incurs instantaneous loss $f_t(\hat{p}_t) = -\ln \left( \hat{p}_t(y_t|x_t) \right)$. At the end of round $t$, the forecaster receives a proxy symbol $z_t$ and a delayed outcome $y_{t-D}$ (if $t - D > 0$).
%\footnote{We consider delayed proxies in our experiments. For our analysis, we consider a fixed outcome delay and no proxy delay for clarity.} 
Figure~\ref{fig:delayed_prediction_timeline} depicts a single prediction round with a delayed outcome. The goal of the forecaster is to minimize the (expected) regret 
\begin{align}
    \E[\regret_T] &= \min_{\omega \in \mathcal{P}} \E\left[\sum_{t=1}^T f_t(\hat{p}_t) - f_t(\omega)\right] \label{eqn:regret}
\end{align}
over $T$ rounds, where the expectation is taken with respect to the proxies $\langle z_t \rangle_{t=1}^T$ and outcomes $\langle y_t \rangle_{t=1}^T$ (note that $\mathcal{P}$ is compact, so the minimum exists). Under Assumption~\ref{asm:almost_adversarial},
the optimal predictor is defined by
\[
p^*_{\rm F}(y|x) = \sum_{z \in \Z} g(y|z)h(z|x)
\]
for all $(x, y) \in \X \times \Y$.

In the usual setups considered in the literature, the forecaster does not have access to the proxies and hence needs to deal with the delays of the outcomes directly.
\citet{Joulani2016} introduce a strategy for converting a base online learner into a delayed online learner. We refer to this approach or any other approach that learns a direct relationship between instances and outcomes as a {\it direct forecaster} for delayed outcomes, since it does not make use of the proxy outcomes. 

According to \citet{Joulani2013}, the regret in adversarial online learning with delay $D$ is at least $\Omega \left( (D+1)N \ln \left( \frac{T}{N(D+1)} \right) \right)$. Applying the approach of \citet{Joulani2016}, it is not difficult to prove that there exists a forecaster with regret $O\left((D+1) N |\Y| \ln \left( \frac{T}{N} + 1 \right) \right)$ (see the supplementary material for a formal statement and proof). Notice that both the lower and upper bounds are logarithmic in the number of rounds $T$ but scale multiplicatively (i.e., $D \times N$) with the delay $D$ and the number of instances $N$. However, this bound is particular to the adversarial setting. If the instances are chosen at each round according to a fixed distribution, then the regret bound scales additively (i.e., $D + N$) \cite{Joulani2013}. Unfortunately, an online marketplace with a stream of new items resembles adversarially generated instances more than independent and identically distributed instances. Our goal is to try to recover an additive scaling between instances and delay even when instances are selected adversarially.
% TODO(agyorgy): Obtain a lower bound for the log loss setting.
% As mentioned in the introduction, a lower bound on the cumulative error in our setting where proxies are ignored is $\Omega(DN + N \ln T)$.

{\bf Proposed Approach:} Given the model described by Assumption~\ref{asm:almost_adversarial}, it is natural to consider a factored model
\begin{align}
    \hat{p}_t(y|x_t) &= \sum_{z \in \Z} \hat{g}_t(y|z) \hat{h}_t(z|x_t) \enspace , \label{eqn:factored_forecaster}
\end{align}
where $\hat{h}$ approximates the conditional distribution $h$ that generates the proxies and $\hat{g}$ estimates the conditional probability $g$ of the delayed outcome given a proxy. The key advantage of a factored model is that the estimator $\hat{h}$ can be updated at each round based on the proxy. While the estimator $\hat{g}$ can only be updated when a delayed outcome is revealed, its predictions do not depend on specific instances. Note that under Assumption~\ref{asm:almost_adversarial}, the optimal predictor $p^*_{\rm F}$ has a factored form.

For simplicity, below we assume that $\hat{g}$ and $\hat{h}$ are Laplace estimators given the outcomes: for $t \ge 1$, given a sequence of observations $\langle a_1,\ldots,a_{t-1}\rangle$ taking values in a finite set $\A$, for the $t$th symbol the Laplace estimator assigns probability $q_t(a)=\frac{\sum_{s=1}^{t-1} \mathbb{I}\{a_s=a\}+\alpha}{t-1+|\A| \alpha}$ for all $a \in \A$ with $\alpha=1$ \citep{Cesa2006}.\footnote{A somewhat better, asymptotically optimal estimator can be obtained by selecting $\alpha=1/2$; this is known as the Krichevsky-Trofimov estimator \citep{Cesa2006}. While our results can easily be applied to this estimator, we choose to present them with the Laplace estimator because it simplifies the form of our bounds.} The resulting algorithm is shown in Algorithm~\ref{alg:dff}, with its regret analyzed in the next theorem.

\begin{algorithm}
\caption{Factored Forecaster for Delayed Outcomes}
\label{alg:dff}
\begin{algorithmic}[1]
\REQUIRE $\X$ \COMMENT{Instance space.}, $\Y$ \COMMENT{Outcome space.}, $\Z$ \COMMENT{Proxy space.}, $D \geq 0$ \COMMENT{Delay in number of rounds.}, $\base$ \COMMENT{Base forecaster.}, $T \geq 1$ \COMMENT{Total number of rounds.}
\STATE Before the game, the environment selects: \\
\begin{enumerate}
    \item instances $\langle x_1, x_2, \dots, x_T \rangle \in \X^T$ arbitrarily;
    \item proxies $\langle z_1, z_2, \dots, z_T \rangle \in \Z^T$ \\ independently with $z_t \sim h(\cdot|x_t)$; and
    \item outcomes $\langle y_1, y_2, \dots, y_T \rangle \in \Y^T$ \\ independently with $y_t \sim g(\cdot|z_t)$.
\end{enumerate}
\STATE For each $x \in \X$, initialize \base{} estimator $\hat{h}^{(x)}$.
\STATE For each $z \in \Z$, initialize \base{} estimator $\hat{g}^{(z)}$.
\FOR{$t = 1, 2, \dots, T$}
\STATE Environment reveals $x_t$.
\STATE Predict $\hat{p}_t$ where \\ $\forall_{y \in \Y} \enspace , \hat{p}_t(y) = \sum_{z \in \Z} \hat{g}^{(z)}(y) \hat{h}^{(x_t)}(z)$.
\STATE Incur loss $f_t(\hat{p}_t) = -\ln \left( \hat{p}_t(y_t) \right)$.
\STATE Environment reveals $z_t$ and $y_{t-D}$ (if $t-D > 0$).
\STATE Update $\hat{h}^{(x_t)}$ with $z_t$ and $\hat{g}^{(z^{t-D})}$ with $y_{t-D}$ using \base (if $t-D > 0$).
\ENDFOR
\end{algorithmic}
\end{algorithm}

\begin{theorem} \label{thm:dff_regret}
Suppose that Assumption \ref{asm:almost_adversarial} holds. Let $D \geq 0$ be the delay applied to outcomes. Then the expected regret of Algorithm~\ref{alg:dff} with a Laplace estimator as \base{} satisfies
\begin{align*}
    \mathbb{E}&\left[ \regret_T \right] \leq \nonumber \\
    O &\bigg( \underbrace{(D+1) |\Y||\Z| \ln \left( \frac{T}{|\Z|} \right)}_{(a)} + \underbrace{|\Z| N \ln \left( \frac{T}{N} \right)}_{(b)} \bigg)~.
\end{align*}
%$p(y_t|x_t) = \sum_{z \in \Z} g(y_t|z)h(z|x_t)$ such that $h$ is defined as in Assumption \ref{asm:almost_adversarial} and $g \in \{ \phi : \Z \rightarrow \Delta^{|\Y|} \}$
\end{theorem}
The proof of the theorem is presented in the supplementary material.
% TODO(kingtim): Describe the theorem. What does it mean? What are the implications? Why does it matter? What does it predict?
There are three terms in the bound of Theorem~\ref{thm:dff_regret}. The first term, (a), depends on the delay $D$ but not on the number of instances $N$ (instead of $N$ it depends on $|\Z|)$. The second term, (b), depends on $N$ but not the delay. Thus, the regret scales additively with respect to the delay and the number of instances, while the regret of the direct forecaster scales multiplicatively (i.e., $\Omega(DN)$).

\begin{figure}
    \centering
    \includegraphics[width=0.45\textwidth]{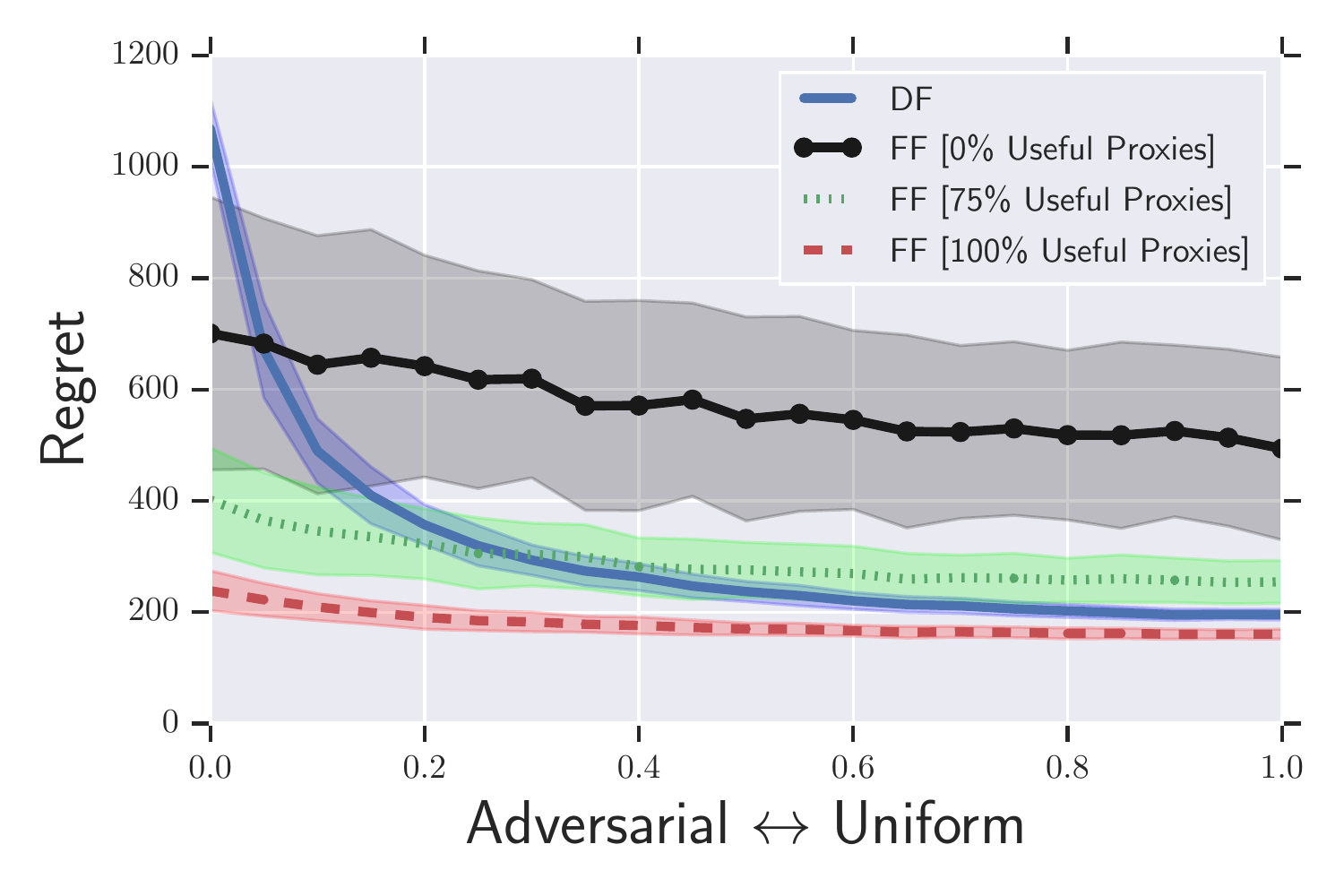}
    \vspace{-0.2cm}
    \caption{Comparison of regret (averaged over 200 trials) on randomly generated prediction tasks for a direct forecaster (DF) and factored forecasters (FFs) as the instances are generated from an adversarial schedule (towards the left) to a uniform schedule (towards the right). Shaded regions denote $\pm1$ standard deviation.}
    \vspace{0.5cm}
    \label{fig:toy_regret}
\end{figure}

Figure \ref{fig:toy_regret} compares the regret of a direct forecaster (DF) and factored forecasters (FFs) in a synthetic prediction domain where the outcomes are generated by an unknown factored model. In this task, $T = 1000$ and $D = 100$. Details of the prediction domain are given in the supplementary material. Towards the left-hand side, the prediction task is generated according to an adversarial schedule that forces the forecaster to make consecutive predictions for the same instance. Towards the right, the schedule is gradually interleaved by a uniform distribution over instances. It is also interesting to consider the case when the factorization assumption (Assumption \ref{asm:almost_adversarial}) is violated. The three FF curves differ by the fraction of the times that the proxy signal is sampled from the true model versus uniform noise. With $100\%$ useful proxies, FF has consistent performance, while DF struggles with the adversarially generated schedule. With $0\%$ useful proxies (random noise), FF incurs large regret relative to DF in most cases. However, the case where FF is given $75\%$ useful proxies, demonstrates that FF can outperform DF in adversarial settings even when the proxy signal is diluted with noise. This suggests that the FF approach may be useful even when we cannot identify high-quality proxies.

\section{Neural Network Architectures for Learning from Delayed Outcomes}
\label{sec:nn}

Based on our analysis, we propose three neural network architectures for learning from delayed outcomes via proxies.

\noindent {\bf Direct forecaster (\DF):} We use a single neural network to predict the distribution over outcomes given an instance (Figure~\ref{fig:input_to_label}). This approach ignores the proxies.

\noindent {\bf Factored forecaster (\FF):} This approach instantiates Algorithm~\ref{alg:dff} with two neural networks (Figure~\ref{fig:input_to_feedback} and~\ref{fig:feedback_to_label}). The first (Figure~\ref{fig:input_to_feedback}) predicts a distribution over proxies given an instance (that is, it implements $\hat{h}_t$), while the second (Figure~\ref{fig:feedback_to_label}) predicts an outcome distribution given a proxy (implementing $\hat{g}_t$). The final prediction is obtained according to \eqref{eqn:factored_forecaster}.

\noindent {\bf Factored with Residual forecaster (\RFF):} Similarly to the factored approach, this method learns two neural networks (Figure~\ref{fig:input_to_feedback} and~\ref{fig:residual}). However, the neural network that predicts an outcome distribution has two towers. The first tower only uses the proxy to predict the logits (implementing $\hat{g}_t$ for the probabilities), while the second tower is an instance-dependent residual correction that can help to correct predictions when the proxies are not informative of the delayed outcomes. Denoting the coordinates of the exponentiated $\Delta$-logits by $1+\delta_t(y|x_t)$ (where $\delta_t$ is typically close to 0), the final prediction is obtained as
$\hat{p}_t(y|x_t)=\sum_{z \in \Z} (1+\delta_t(y|x_t,z)) \hat{g}_t(y|z)\hat{h}_t(z|x_t)$.
We train the residual tower with a separate loss and stop backpropogation from that loss to the first tower. This ensures that the second tower is treated as a residual correction helping to preserve generalization across instances.

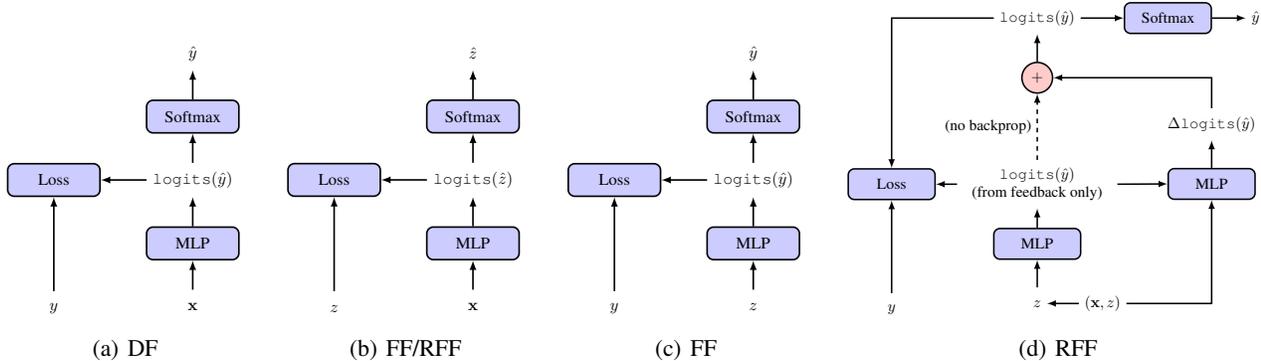
\begin{figure*}
    \centering
    \subfigure[][\DF]{\resizebox{0.19\textwidth}{!}{
        \input{figures_models_input_to_label}
        \label{fig:input_to_label}}}
        \;\;\;
    \subfigure[][\FF/\RFF]{\resizebox{0.19\textwidth}{!}{
        \input{figures_models_input_to_feedback}
        \label{fig:input_to_feedback}}}
        \;\;\;
    \subfigure[][\FF]{\resizebox{0.19\textwidth}{!}{
        \input{figures_models_feedback_to_label}
        \label{fig:feedback_to_label}}}
        \;\;\;
    \subfigure[][\RFF]{\resizebox{0.3325\textwidth}{!}{
        \input{figures_models_residual}
        \label{fig:residual}}}
    \caption{Neural network architectures for delayed prediction, where $x$ denotes an instance, $y$ denotes an outcome, $\hat{y}$ denotes a predicted outcome distribution, $z$ denotes a proxy, and $\hat{z}$ denotes a predicted distribution over proxies. The direct model~\subref{fig:input_to_label} directly predicts $\hat{y}$ an outcome distribution from an instance $x$. The factored model is defined by two neural networks~\subref{fig:input_to_feedback} and~\subref{fig:feedback_to_label}, where ~\subref{fig:input_to_feedback} predicts $\hat{z}$ the proxy distribution given an instance $x$ and~\subref{fig:feedback_to_label} predicts $\hat{y}$ an outcome distribution from an proxy $z$. Finally, the factored with residual model is defined by two networks~\subref{fig:input_to_feedback} and~\subref{fig:residual}, where ~\subref{fig:residual} uses both an instance $x$ and proxy $z$ to predict $\hat{y}$ an outcome distribution.}
    \label{fig:delayed_arcs}
\end{figure*}

For all experiments, unless stated otherwise, we update network weights using Stochastic Gradient Descent\footnote{We also tried other optimizers but found that most were not sensitive enough to changes in the distribution over instances as they keep track of a historical average over gradients.} with a learning rate of $0.1$ minimizing the negative log-loss. Except for the networks predicting the outcome distribution from proxies (Figure~\ref{fig:feedback_to_label} and left tower in Figure~\ref{fig:residual}), all network towers have two hidden layers. Their output layer is sized appropriately (to output the correct number of class logits) and use a \texttt{softmax} activation. We apply $L_2$ regularization on the weights with a scale parameter of $0.01$. The networks predicting the outcome distribution from proxies do not contain hidden layers, they do not use regularization or bias in their output layer. Training samples are stored in a fixed-sized FIFO replay buffer from which we sample uniformly.

For the experiment explained in Section~\ref{sec:github}, the networks predicting the outcome distribution from proxies use a learning rate of $1$. Network towers have two hidden layers with 40 and 20 units. The training buffer has a size of 1,000. We start training once we have 128 examples in the buffer and perform one gradient step with a batch size of 128 every four rounds.

% For the experiment explained in Section~\ref{sec:app_store},
For the experiment explained in Section~\ref{sec:store},
the networks predicting the outcome distribution from proxies use a learning rate of $0.1$. Network towers have two hidden layers with 20 and 10 units. The training buffer has a size of 3,000. We start training once we have 500 examples in the buffer and perform 20 gradient steps with a batch size of 128 every 1,000 rounds.\footnote{To simulate less frequent updates to the networks.} 
 
In both experiments, hyperparameters such as the learning rate and the coefficient of the $L_2$ regularization were tuned on \DF first, using a simple grid search. The hyperparameters that do not relate to \DF (such as the number of hidden layers and the $L_2$ regularization coefficient of the networks predicting the proxies) are tuned in a second step using a simple grid search, as well.
%The parameters were tuned using grid search.

\section{Experiments \& Results}
\label{sec:exp}

We compare the regret of \DF, \FF, and \RFF in two domains: (1) predicting the commit activity of GitHub repositories, and (2) predicting engagement with items acquired from a marketplace.

\subsection{GitHub Commit Activity \label{sec:github}}

\begin{figure*}
    \centering
    \includegraphics[width=0.95\textwidth]{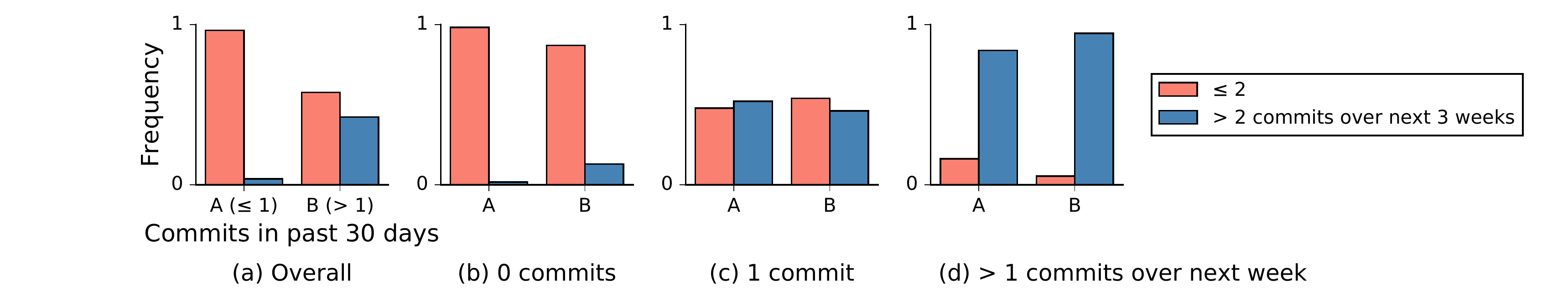}
    \caption{Comparison of outcome distributions for each proxy under different adversarial modes (A and B). (a) shows the change in overall outcome distributions between the two modes, while~(b-d) show the marginal distributions conditioned on each proxy.}
    \label{fig:label_distributions}
\end{figure*}

The goal is to predict the number of commits made to repositories from GitHub\footnote{\url{http://www.github.com}} - a popular website hosting open source software projects. Given a repository, the question we want the online learner to answer is ``will there be at least three commits in the next three weeks?''. This information could be used to predict churn rate. For example, GitHub could potentially intervene by sending a reminder email. We obtained historical information about commits to GitHub repositories from the BigQuery GitHub database.\footnote{\url{https://cloud.google.com/bigquery/public-data/github}} We started with 100,000 repositories and filtered out repositories with fewer than five unique days with commits between May 1, 2017 and January 8, 2018. This resulted in about 8,300 repositories. In our experiments, an adversary selects both a repository and a timestamp. The outcome is one if there were at least three commits over the 21 days following the chosen timestamp and zero otherwise. The proxy is based on the number of commits over the first seven days following the chosen timestamp. These were mapped to three values: (1) no commits, (2) one commit, and (3) more than one commits.
% How much delay in the outcome and intermediate observation
The proxy is delayed by one week and the outcomes are delayed by three weeks. We simulate a scenario where the forecaster receives one sample every 10 minutes.

% The adversary
The adversary initially selects repositories from a subset with a low number of commits - one or fewer commits over the past month. After four and a half weeks, the adversary switches to a distribution that samples from repositories with two or more commits within the past week. Figure~\ref{fig:label_distributions} shows the outcome probabilities for each proxy under the two distributions used by the adversary. The outcome distribution does not have a fixed relationship with the proxies. Thus, this makes the task more difficult for the factored architectures but the direct architecture is not affected.
% What features did we use to specify instances?
In this experiment, the historical information about a repository defines an instance. We used binary features to represent the programming languages present in a repository as well as time bucketized counts of historical commit activity for that repository.

% Measuring cumulative error
To calculate the regret, we subtract the loss of an (almost) optimal forecaster. Since we do not have access to an optimal forecaster, we trained two models with the same architecture as \DF on both modes used by the adversary, on non-delayed data. We trained these models using the \texttt{Adam} optimizer for 10,000 steps and using an initial learning rate of 0.0005.

Figure~\ref{fig:regret_with_shifting_adversary_loss} compares the loss of the direct and factored architectures averaged over 200 independent trials. The vertical dashed line indicates the time at which the adversary switches from its initial distribution to a distribution over high commit repositories. Due to the outcome delay, all algorithms suffer the same initial loss for roughly 3 weeks. Then all three algorithms quickly achieve a low loss until the adversary changes the distribution over repositories. Finally, {\em the factored architectures (\FF and \RFF) recover roughly 2.5 weeks more quickly than \DF} (as can be seen in Figure~\ref{fig:regret_with_shifting_adversary_loss} from week 5.5--8). Figure~\ref{fig:regret_with_shifting_adversary_regret} shows the regret for the same experiment. \FF and \RFF achieve smaller regret (than \DF) as they are better able to adapt to changing distributions. We can also clearly observe that \FF is not able to maintain as low loss as \RFF, since it is unable to correct the error coming from the slightly incorrect factorization assumption.

\begin{figure*}
    \centering
    \subfigure[][]{
        \includegraphics[width=0.4\textwidth]{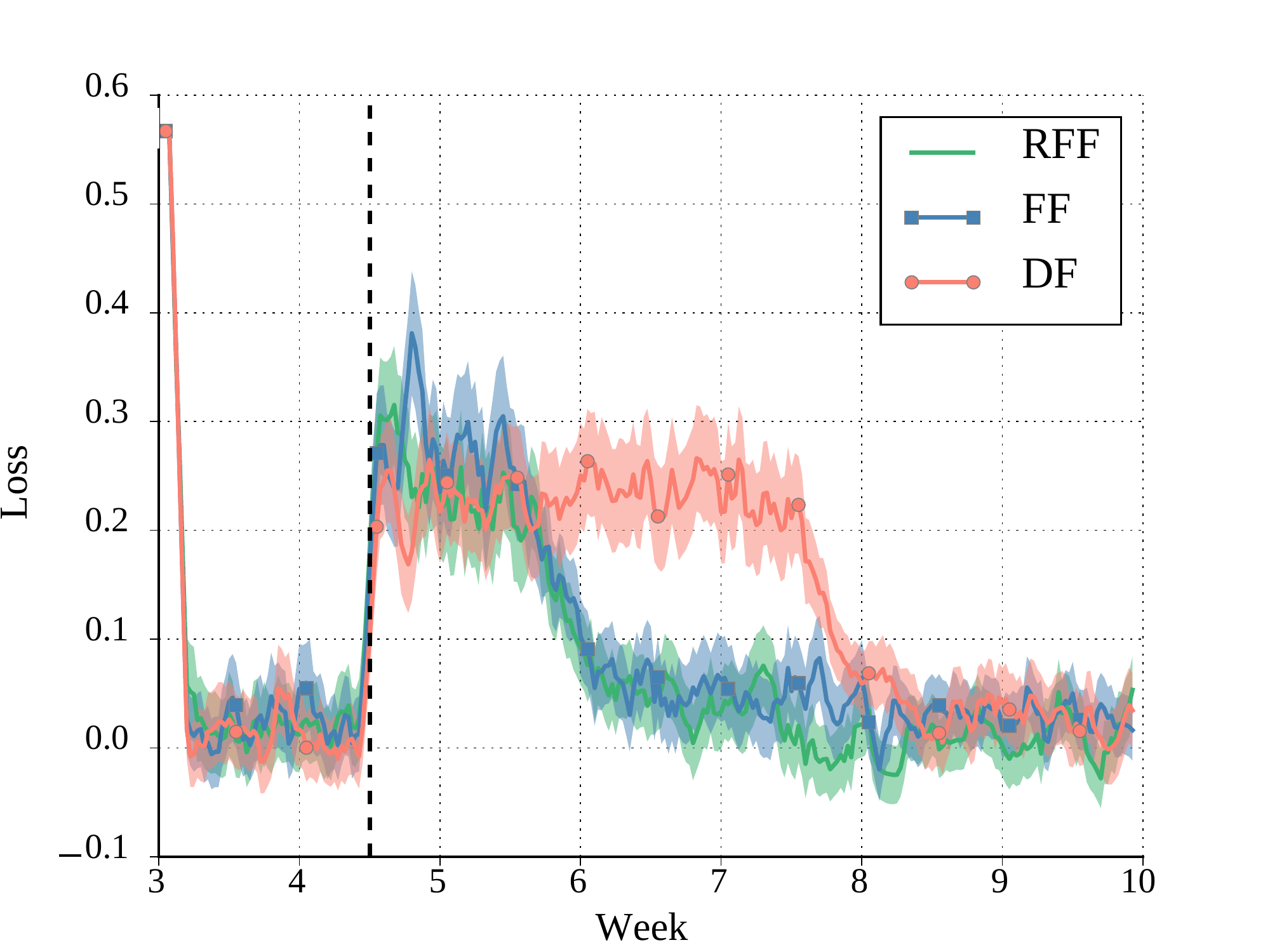}
        \label{fig:regret_with_shifting_adversary_loss}}
    \subfigure[][]{
        \includegraphics[width=0.4\textwidth]{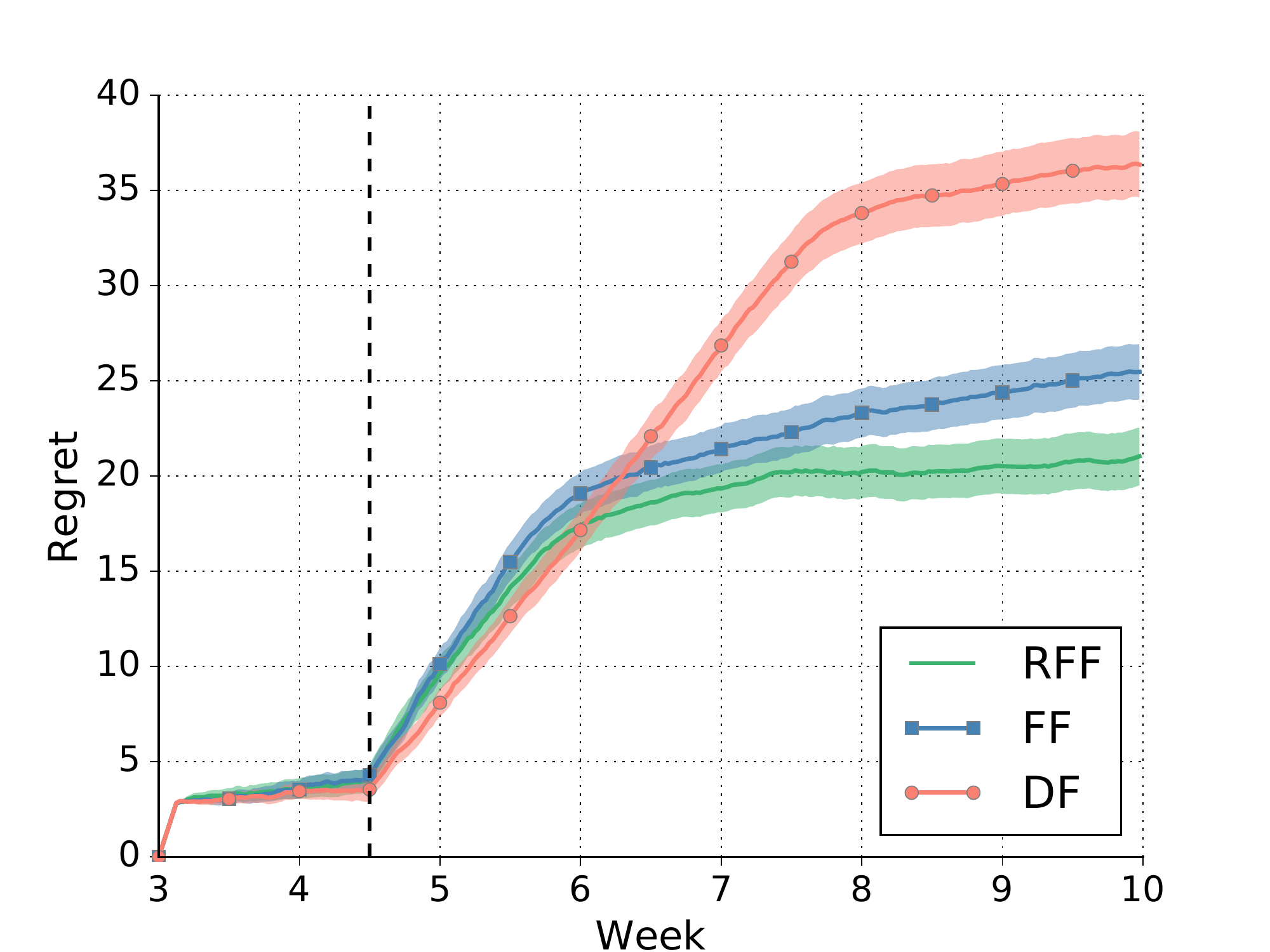}
        \label{fig:regret_with_shifting_adversary_regret}}
    \caption{Comparisons of~\subref{fig:regret_with_shifting_adversary_loss} average loss and~\subref{fig:regret_with_shifting_adversary_regret} regret (averaged over 200 independent trials with \subref{fig:regret_with_shifting_adversary_loss} further averaged over five consecutive datapoints) in a task where the adversary shifts the distribution of instances partway through the episode (indicated by the dashed vertical line). We measure regret with respect to an optimally trained model. The shaded areas represent 95\% confidence intervals.}
    \label{fig:regret_with_shifting_adversary}
\end{figure*}

\subsection{Engagement with Marketplace Items \label{sec:store}}

For a popular marketplace with personalized recommendations, we want to predict the probability that after acquiring an item a user will (1) engage with that item more than once, and (2) not delete that item within 7 days after acquiring it (e.g., a user who downloads a new ebook will read at least two chapters and not delete this ebook from their device). The proxies are measured two days after an item is acquired and the possible outcomes are: (1) {\it Deleted:} The item was deleted. (2) {\it Zero Engagements:} The user engaged with the item zero times but did not delete it. (3) {\it One Engagement:} The user engaged with the item exactly once and did not delete it. (4) {\it Many Engagements:} The user engaged with the item two or more times and did not delete it. The outcomes are the same as the proxies but measured seven days after conversion.

\begin{figure*}
\centering
\subfigure[][Staggered Schedule]{
    \includegraphics[width=0.4\textwidth]{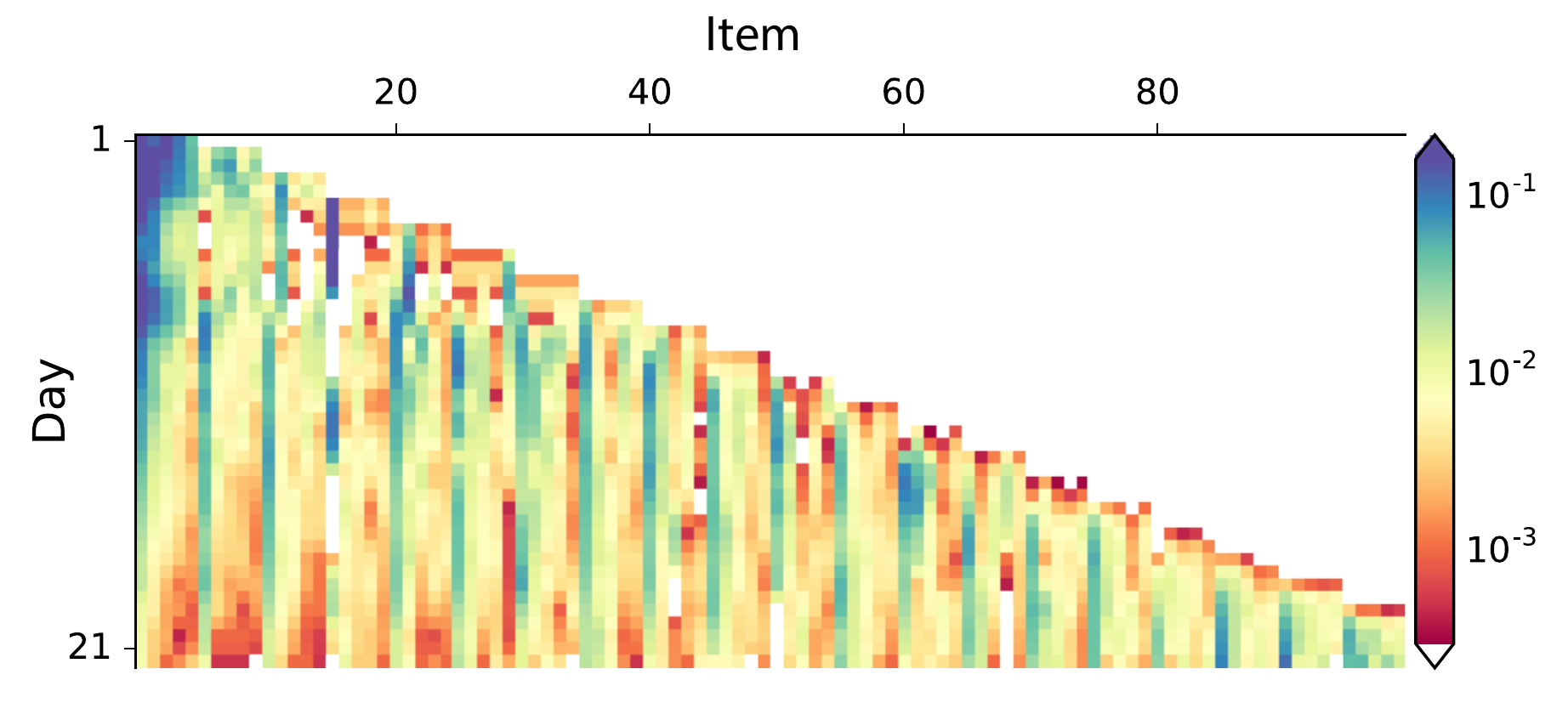}
    \label{fig:store_adversaries_staggered}}
\subfigure[][Uniform Schedule]{
    \includegraphics[width=0.4\textwidth]{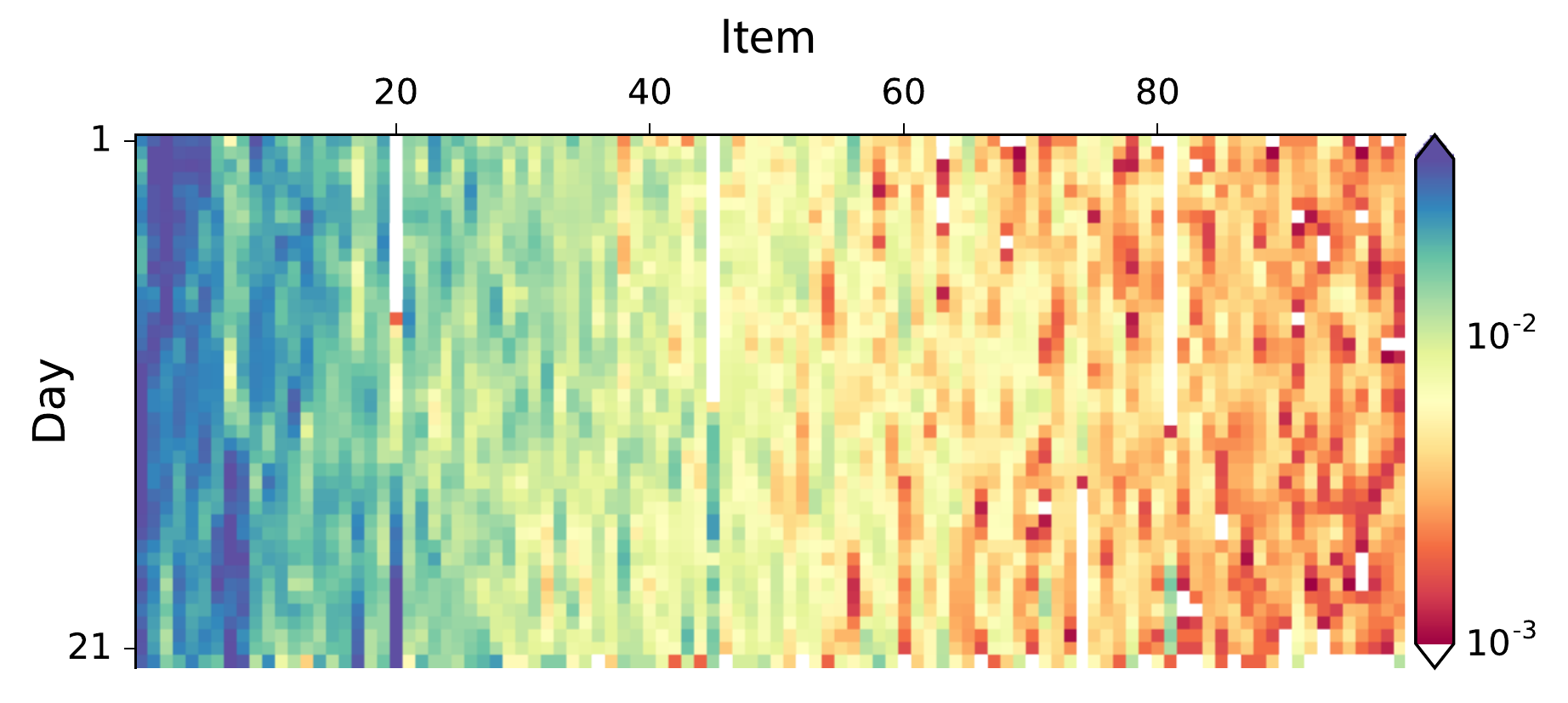}
    \label{fig:store_adversaries_uniform}}
\caption{Two different item distribution schedules derived from an actual marketplace. The X-axis corresponds to 100 item indices sorted in decreasing order of total conversions, while the Y-axis corresponds to days. \subref{fig:store_adversaries_staggered} Empirical distribution (logarithmic scale) with new items being added each day. \subref{fig:store_adversaries_uniform} More stable empirical distribution of uniformly sampled items.}
\label{fig:store_adversaries}
\end{figure*}

\begin{figure*}
\centering
\subfigure[][]{
    \includegraphics[width=0.4\textwidth]{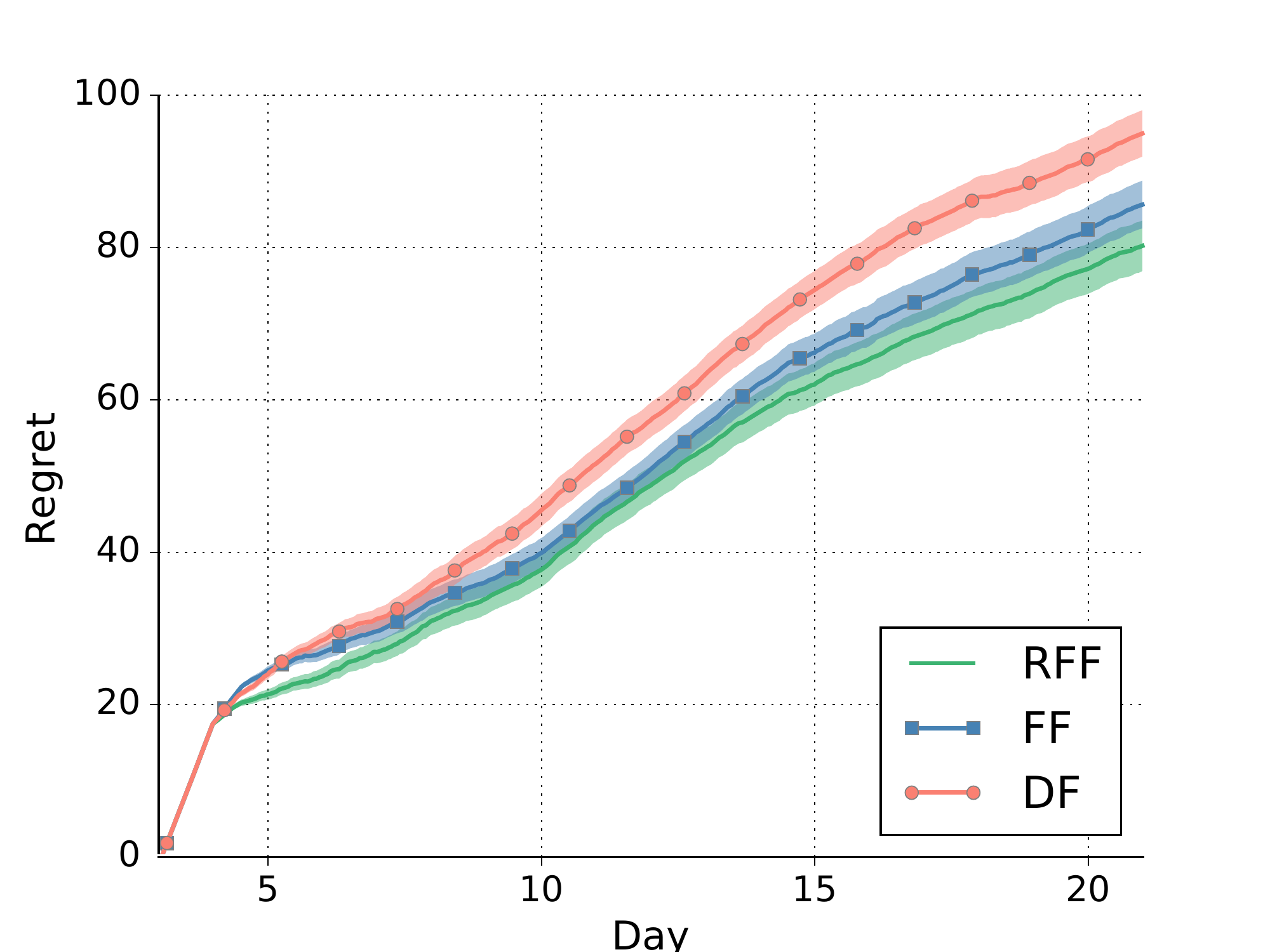}
    \label{fig:store_regret_staggered}}
\subfigure[][]{
    \includegraphics[width=0.4\textwidth]{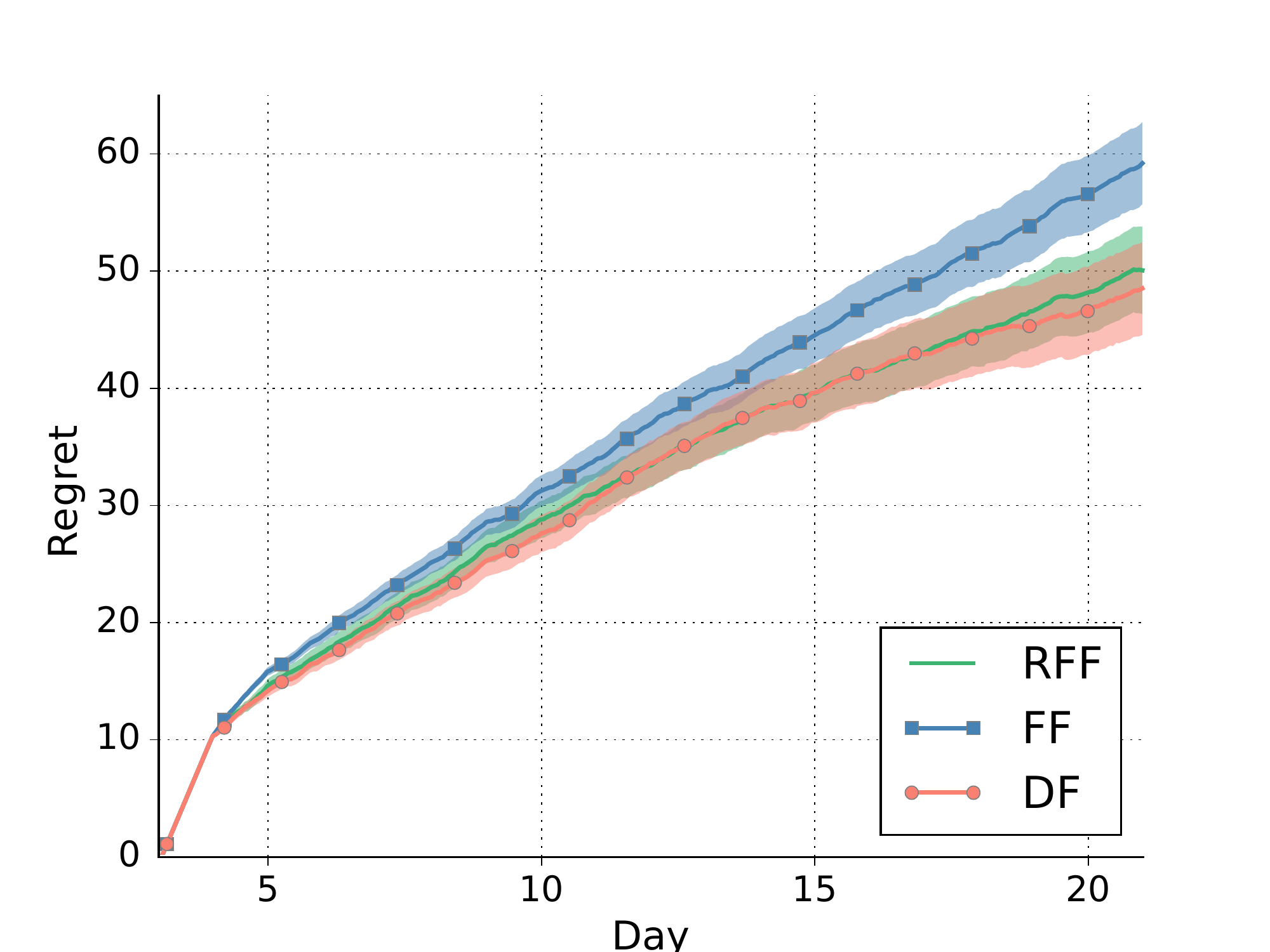}
    \label{fig:store_regret_uniform}}
\caption{Regret (and 95\% confidence intervals) of \DF, \FF, and \RFF averaged over 200 independent trials on items sampled from different schedules. \subref{fig:store_regret_staggered} In the staggered schedule, the factored approaches outperform \DF, which is crippled by delay. \subref{fig:store_regret_uniform} The uniform schedule demonstrates how \RFF can achieve similar performance to \DF, when the factorization assumption does not help. (Note the difference in the scale of the Y-axis).}
\label{fig:store_regret}
\end{figure*}

We collected 21 days of data between the 10th and 31st of January 2018. For each item with at least 100 conversions, we stored the empirical probability of each proxy and the probability of each outcome conditioned on proxies. In addition, we stored the empirical probability that an item was acquired (if shown) on each of the 21 days. We consider two distribution schedules (Figure \ref{fig:store_adversaries}). The first schedule, which we refer to as Staggered (Figure~\ref{fig:store_adversaries_staggered}), subsampled 5 items for each day (at noon) that appear for the first time on that day in the marketplace (i.e., the candidate items with higher indices only appear on increasingly later days). This schedule simulates items continually being added to the marketplace. This is a very hard setting in which we expect \DF to perform poorly, due to the need to constantly make predictions about new instances. The second schedule, which we refer to as Uniform, is derived by subsampling 100 items uniformly (Figure~\ref{fig:store_adversaries_uniform}) and creating a 21 day schedule with half-day intervals (for the afternoons, the distributions are obtained as the average distributions for the actual and the next days, based on the frequencies from the 21 logged days). This schedule is more favorable for \DF because instances are sampled according to a slowly shifting distribution.
% How much delay in the outcomes and the intermediate observations
Overall, the proxies are delayed by two days and the outcomes are delayed by seven days. We simulated a scenario where the forecaster receives one sample every 40 seconds.

In this experiment, we encoded an item/instance using a 100-dimensional one-hot encoding specifying items by index.
% Measuring cumulative error
Similarly to the GitHub experiment, to calculate the regret, we subtract the loss of an (almost) optimal forecaster. We trained 42 models (one for each half-day interval) with the same architecture as \DF on all modes used by the adversary (one model for each half-day interval). For each model, we used the \texttt{Adam} optimizer and trained for 40,000 steps using non-delayed data, with an initial learning rate of 0.0005.

Figure \ref{fig:store_regret} compares the average regret of all three forecasters over 200 independent trials for both adversarial schedules.
For the Staggered schedule (Figure \ref{fig:store_regret_staggered}), \FF and \RFF outperform \DF, as expected. \DF must wait until the outcomes become available, but \FF and \RFF are able to generalize to new instances. \RFF performs slightly better than \FF, indicating that \RFF can mitigate the incorrect factorization assumption as delayed outcomes become available.

The Uniform schedule (Figure \ref{fig:store_regret_uniform}) is easier since the distribution over instances is shifting slowly. \DF, \FF, and \RFF all achieve smaller regret compared to the Staggered schedule. In the Uniform schedule, \DF outperforms \FF because the factorization assumption is violated by the data. However, \RFF achieves similar results to \DF because its residual tower allows it to mitigate the error introduced by the incorrect factorization assumption.

\section{Related Work}
\label{sec:rel}

%\rednote{a bit of repetition in related work and discussion; these could be merged to save space}

\citet{Chapelle2014} proposes a model for learning from delayed conversions. However, this approach does not take advantage of potential proxies for delayed outcomes. % Survival analysis
Learning from delayed outcomes is also related to survival analysis \citep{Yu2011,Fernandez2016}, where the goal is to model the time until a delayed event. A significant difference of our work is the use of proxies.

A large body of literature exists in online learning for both the adversarial and stochastic partial monitoring settings with delayed outcomes, analyzing the regret (see, e.g.,
\citealp{Weinberger2002,Mesterharm2005,%langford2009slowa,
agarwal2011distributed,Joulani2013,Joulani2016}). However, none of the settings takes potential proxies  into consideration to mitigate the impact of delay.

\citet{Cesa2018} consider an adversarial learning problem with delayed bandit-based feedback. In their setting, the observed loss function is the sum of $D$ adversarially chosen loss functions for the actions taken in the previous $D$ rounds. This differs significantly from our setting where proxy signals are {\em additional} information to the forecaster. Furthermore, our analysis does not restrict the proxies to a particular relationship with the delayed outcomes.

In a fully stochastic online learning setting, where the instances and outcomes are sampled from the same distribution at each round, the regret \citep{Joulani2013} or mistake bounds \citep{Mesterharm2005} scale with $D+N$ (i.e., delay plus number of instances; the exact dependence on the time horizon $T$ depends on the loss function). However, the stochastic assumption is not realistic for online marketplaces because new products are being added on a regular basis. Thus, the distribution over instances is not independent and identically distributed from day to day.

\section{Discussion}
\label{sec:discuss}

We have presented a way to leverage proxies to learn faster in scenarios where the outcomes are significantly delayed. 
Our theoretical analysis shows that under our factorization assumption, the regret of the factored approach which exploits intermediate observations scales as $D+N$ unlike a naive approach that scales as $D\times N$.

While we have presented instances and proxies as disjoint sets, note that in practice, we can encode some information from an instance into the proxy. For example, in the motivating example, an instance corresponds to a book. From the instance, we can extract features that may generalize, such as the length of the ebook. A proxy could encode both how much the customer has read so far and the length of the ebook. This technically increases the number of proxy symbols, but could significantly improve the predictive power of a factored forecaster.

While we have focused on full-information online learning, it may also be interesting to consider learning with bandit feedback. However, applying bandit algorithms in a production setting can be challenging, since decisions are typically made by several components (not a single agent). Nevertheless, this is an interesting direction of future work.

% We present algorithms that can exploit
We have presented experimental results on a dataset from GitHub as well as a dataset from a real marketplace, and showed that our algorithms can learn faster when the chosen proxies are helpful, and can gracefully recover the baseline performance when the selected proxies are not informative of the delayed outcomes.  
We believe that the proposed approach can be beneficial in many real-world applications where the goal is to optimize for long-term value. It would be interesting to extend our theoretical analysis to ranking measures.

\section*{Acknowledgements}

We would like to thank David Silver for helpful discussions regarding learning from delayed signals and Tor Lattimore and Todd Hester for reviewing this manuscript.

\bibliographystyle{icml2019}
\bibliography{main}

\clearpage
%\end{document}
\appendix
\onecolumn

\section{Analysis of the Direct Forecaster}

\begin{algorithm}
\caption{Direct Forecaster for Delayed Outcomes}
\label{alg:ddf}
\begin{algorithmic}[1]
\REQUIRE $\X$ \COMMENT{Instance space.}, $\Y$ \COMMENT{Outcome space.}, $D \geq 0$ \COMMENT{Delay in \# rounds.}, $\base$ \COMMENT{Base forecaster.}, $T \geq 1$ \COMMENT{Total rounds.}
\STATE Before the game, the environment selects: \\
\begin{enumerate}
    \item instances $\langle x_1, x_2, \dots, x_T \rangle \in \X^T$ arbitrarily, and
    \item instances $\langle y_1, y_2, \dots, y_T \rangle \in \Y^T$ arbitrarily.
\end{enumerate}
\STATE For each $x \in \X$, initialize $\base$ estimator $\hat{p}^{(x)}$.
\FOR{$t = 1, 2, \dots, T$}
\STATE Environment reveals $x_t$.
\STATE Predict $\hat{p}_t = \hat{p}^{(x_t)}$.
\STATE Incur loss $f_t(\hat{p}_t) = - \ln \left( \hat{p}_t(y_t) \right)$.
\STATE Update $\hat{p}^{(x_{t-D})}$ with $y_{t-D}$ (if $t-D > 0$).
\ENDFOR
\end{algorithmic}
\end{algorithm}

In this section we formally introduce and analyze a direct forecaster. The forecaster, given in Algorithm~\ref{alg:ddf}, creates a separate estimator for each instance $x \in \X$. Thus, it converts the problem into $|\X|$ subproblems. 
Similarly to the factored forecaster, we consider the case when the \base{} forecasters (for each $x \in \X$) are of the form 
\begin{align}
\label{eqn:KT}
    \hat{p}_t (y|x) &= \frac{\sum_{s=1}^{t-1} \mathbb{I}\{x_s = x \wedge y_s = y\} + \alpha}{\sum_{s=1}^{t-1} \mathbb{I}\{x_s = x\} + \alpha |\Y|}~,
\end{align}
for some $\alpha \ge 0$ where the sums are defined to be $0$ if $t \le 1$ (notable special cases are the Krichevsky-Trofimov estimator for $\alpha=1/2$ and the Laplace estimator for $\alpha=1$, see, e.g., \citealt{Cesa2006}). For $t \le 0$, $\hat{p}_t(y|x)$ is defined to be $1$ for any pair $(x,y) \in \X \times \Y$.

Following \citet{Joulani2016}, we characterize the effect of the delayed observations by considering the so-called prediction drift of the non-delayed forecaster \base{} given in \eqref{eqn:KT}, which looks at the loss difference of the predictions of the non-delayed and the delayed forecasters, $\hat{p}_s(y_s|x_s)$ and $\hat{p}_{s-D}(y_s|x_s)$, respectively.

\begin{lemma} \label{lem:log_prediction_drift_bound}
Let $\alpha > 0$ and $D \geq 0$. Then for any sequence $\langle (x_s, y_s) \rangle_{s=1}^T \in (\X \times \Y)^T$, predictor \eqref{eqn:KT}
satisfies
\begin{align}
    \sum_{s=1}^T \ln \left( \frac{\hat{p}_{s}(y_s|x_s)}{\hat{p}_{s-D}(y_s|x_s)} \right) & \le 
    D|\X||\Y| \ln\left(\frac{T-1}{\alpha |\X||\Y|} + 1\right)
    %
    %|\X||\Y| \left(D \ln\left(\frac{T-D-1}{|\X||\Y|} + D + %\alpha\right) + \ln(1/\alpha) \right)~. 
    \label{eqn:log_prediction_drift_bound}
\end{align}
\end{lemma}
The expression on the left-hand side of \eqref{eqn:log_prediction_drift_bound} is called the cumulated prediction drift.
\begin{proof}
For any $t$, define $\#_t(x) = \sum_{s=1}^{t-1} \mathbb{I}\{x_s = x \}$
and $\#_t(y,x) = \sum_{s=1}^{t-1} \mathbb{I}\{y_s = y \wedge x_s = x \}$, where the sums are defined to be $0$ if $t \le 1$.
Then prediction drift can be bounded as follows:
\begin{align}
    \sum_{s=1}^T \ln \left( \frac{\hat{p}_{s}(y_s|x_s)}{\hat{p}_{s-D}(y_s|x_s)} \right) 
    &= \sum_{s=1}^T \ln \left(\frac{\frac{\#_s(y_s,x_s)+\alpha}{\#_s(x_s)+\alpha|\Y|}}{\frac{\#_{s-D}(y_s,x_s)+\alpha}{\#_{s-D}(x_s)+\alpha|\Y|}}\right)
     \le \sum_{s=1}^T \ln \left(\frac{\#_s(y_s,x_s)+\alpha}{\#_{s-D}(y_s,x_s)+\alpha} \right) \nonumber \\
    & = \sum_{(x,y) \in \X \times \Y} \sum_{s=1}^{\#_{T}(x,y)} \Big(\ln(\#_s(y,x)+\alpha) - \ln(\#_{s-D}(y,x)+\alpha) \Big) 
    \nonumber \\
    & \le
     \sum_{(x,y) \in \X \times \Y} D \Big( \ln(\#_T(x,y)+\alpha) - \ln(\alpha)\Big)
     = \sum_{(x,y) \in \X \times \Y} D  \ln \left(\frac{\#_T(x,y)}{\alpha}+1\right)~,
    \label{eqn:pred_drift_first_bound1}
\end{align}
where the first inequality holds since $\#_{s-D}(x_s) \le \#_s(x_s)$, while the second inequality holds since at most the last $D$ positive and the first $D$ negative terms are not canceled in the middle row. By Jensen's inequality, using the concavity of $\ln(\cdot)$ and that $\sum_{(x,y) \in \X \times \Y} \#_{T}(x,y) = T-1$, \eqref{eqn:pred_drift_first_bound1} can be bounded from above as
\begin{align}
    \eqref{eqn:pred_drift_first_bound1}
    & \le D|\X||\Y| \ln\left(\frac{T-1}{\alpha |\X||\Y|} + 1\right)
\end{align}
finishing the proof of the lemma.
\end{proof}

The following theorem bounds the regret of Algorithm~\ref{alg:ddf} with Laplace estimators as $\base$. The analysis follows that of \citet{Joulani2016} and is presented for completeness.

\begin{theorem} \label{thm:df_regret_bound}
Let $D \geq 0$ be the delay applied to outcomes, $\X$ be the set of instance symbols, $\Y$ be the set of possible outcomes, and $\langle (x_1, y_1), (x_2, y_2), \dots, (x_T, y_T) \rangle$ be an arbitrary sequence of instance and outcome pairs. For any competitor $p \in \{ \phi \mid \phi : \X \rightarrow \Delta^{|\Y|} \}$, the regret of Algorithm \ref{alg:ddf} with a Laplace estimator as $\base$ satisfies
\begin{align}
    \regret_T(p) = \sum_{t=1}^T \ln \left( \frac{p(y_t|x_t)}{\hat{p}_{t-D}^{(x_t)}(y_t)} \right) \leq O \left( (D+1) |\Y| |\X| \ln \left( \frac{T}{|\X|} + 1 \right) \right)~.
\end{align}
\end{theorem}
\begin{proof}
Let $R(T, \Y)$ be a concave function upper bounding the regret of $\base$ applied to an arbitrary sequence of length $T$ with symbols from $\Y$. First we rewrite the regret of the algorithm with delayed feedback as the sum of the non-delayed regret and the prediction drift:
\begin{align*}
    \regret_T(p) &= \sum_{t=1}^T f_t(\hat{p}_{t-D}) - f_t(p) 
    = \sum_{t=1}^T \ln \left( \frac{p(y_t|x_t)}{\hat{p}_{t-D}^{(x_t)}(y_t)} \right) \\
    &= \underbrace{\sum_{t=1}^T \ln \left( \frac{p(y_t|x_t)}{\hat{p}_t^{(x_t)}(y_t)} \right)}_{(a)}\, +\, \underbrace{\sum_{t=1}^T \ln \left( \frac{\hat{p}_t^{(x_t)}(y_t)}{\hat{p}_{t-D}^{(x_t)}(y_t)} \right)}_{(b)}~.
\end{align*}
Term (b) can be bounded by Lemma~\ref{lem:log_prediction_drift_bound}. Next, we consider term (a). Writing it as the sum of regret over prediction tasks for each instance $x \in \X$, we obtain
\begin{align*}
    \sum_{t=1}^T \ln \left( \frac{p(y_t|x_t)}{\hat{p}_t^{(x_t)}(y_t)} \right) &= \sum_{x \in \X} \sum_{t=1}^{T} \mathbb{I}\{x = x_t\} \ln \left( \frac{p(y_t|x)}{\hat{p}^{(x)}_t(y_t)} \right) %&& \textrm{Regret of } |\X| \textrm{ prediction tasks.} 
    \\
    &\leq \sum_{x \in \X} R\left(\sum_{t=1}^T \mathbb{I}\{x = x_t\}, \Y \right) %&& \textrm{Since } R \textrm{ is an upper bound on the regret for each individual } x \in \X . 
    \\
    %&= |\X| \sum_{x \in \X} \frac{1}{|\X|} R\left(\sum_{t=1}^T \mathbb{I}\{x = x_t\}, \Y \right) && \\ %\textrm{Since } \frac{|\X|}{|\X|} = 1. \\
    %
    &\leq |\X| R \left( \frac{1}{|\X|} \sum_{x \in \X} \sum_{t=1}^T \mathbb{I}\{x = x_t\}, \Y  \right) %&& \textrm{Applying Jensen's inequality to the concave function $R$.} 
    \\
    &= |\X| R\left( \frac{T}{|\X|}, \Y \right)~,
\end{align*}
where the first inequality follows since $R$ is an upper bound on the regret of each individual prediction task, while the second inequality follows from applying Jensen's inequality to the concave function $R$.
By Remark 9.3 in \citet{Cesa2006}, for the Laplace and the Krichevsky-Trofimov estimator, $R(T, \Y) = O \left( |\Y| \ln T \right)$. Thus, term (a) is bound by $O\left( |\X| |\Y| \ln \left( \frac{T}{|\X|} \right) \right)$. Combining with the bound on term (b) finishes the proof.

%By Lemma \ref{lem:log_prediction_drift_bound} the term (b) is bound $(D+1)|\Y||\X| \ln \left( \frac{T}{|\Y||\X|} + 2 + D \right)$. Thus $\regret_T(p)$ is bound by $O\left( (D+1) |\Y||\X| \ln \left( \frac{T}{|\X|} + D \right) \right)$, which proves the result.
\end{proof}

\section{Proof of Theorem \ref{thm:dff_regret}}

\newcommand{\pf}{p^*_{\rm F}}

In this section we prove Theorem~\ref{thm:dff_regret}. We start with presenting an intermediate bound on the regret, which allows us to directly use regret bounds available for the \base algorithm, given in Theorem~\ref{thm:df_regret_bound}.

\begin{lemma} \label{lem:factored_regret}
Suppose Assumption~\ref{asm:almost_adversarial} holds. Then
the expected regret of Algorithm~\ref{alg:dff} can be bounded as
\begin{align}
\label{eqn:decomp}
    \mathbb{E} \left[ \regret_T \right] \le \mathbb{E} \left[  \sum_{t=1}^T \ln \left( \frac{g(y_t|z_t)}{\hat{g}_t(y_t|z_t)} \right) \right] + \mathbb{E} \left[ \sum_{t=1}^T \ln \left( \frac{h(z_t|x_t)}{\hat{h}_t(z_t|x_t)} \right) \right]~.
\end{align}
\end{lemma}
\begin{proof}
By definition,
\begin{align*}
    \mathbb{E} \left[ \regret_T \right] 
    &= \mathbb{E} \left[ \sum_{t=1}^T \ln \left( \frac{\sum_{z' \in \Z} g(y_t|z') h(z'|x_t)}{\sum_{z \in \Z} \hat{g}_t(y_t|z) \hat{h}_t(z|x_t)} \right) \right] \\
    &= \mathbb{E} \left[ \sum_{t=1}^T \frac{1}{\pf(y_t|x_t)} \left(\sum_{z \in \Z} g(y_t|z) h(z|x_t)\right)\ln \left( \frac{\sum_{z \in \Z} g(y_t|z) h(z|x_t)}{\sum_{z \in \Z} \hat{g}_t(y_t|z) \hat{h}_t(z|x_t)} \right) \right]~.
\end{align*}
Applying the log-sum inequality \citep[Theorem~2.7.1]{CoTh06}\footnote{For any non-negative numbers $a_1,\ldots,a_T$ and $b_1,\ldots,b_T$, 
$\left(\sum_{i=1}^T a_i\right) \ln \frac{\sum_{i=1}^T a_i}{\sum_{i=1}^T b_i} \le \sum_{i=1}^T a_i \ln \frac{a_i}{b_i}$, where, by convention, $0\ln 0=0$, $a\ln\tfrac{a}{0}=\infty$ if $a>0$ and $0\ln\tfrac{0}{0}=0$.}
for the inner summations over $z$, we obtain
\begin{align}
    \mathbb{E}_S \left[ \regret_T \right] &\leq \mathbb{E}_S \left[ \sum_{t=1}^T \sum_{z \in \Z} \frac{g(y_t|z)h(z|x_t)}{\pf(y_t|x_t)} \ln \left( \frac{g(y_t|z)h(z|x_t)}{\hat{g}_t(y_t|z)\hat{h}_t(z|x_t)} \right) \right]
    \label{eqn:logsum}
\end{align}
Now define $z'_t$ conditionally independently of the other variables $\langle x_s,y_s,z_s, z'_s\rangle_{s\neq t}$ via the conditional distribution  
$\Pr[z'_t=z|x_t,y_t]=\frac{g(y_t|z)h(z|x_t)}{\pf(y_t|x_t)}$. Then, under our assumptions, $z_t$ and $z'_t$ have the same conditional distribution given $x_t$ and $y_t$, and so by the independence of $z_t$ and $\hat{g}_t$ and $\hat{h}_t$, 
\begin{align*}    %
    \eqref{eqn:logsum} &= \mathbb{E} \left[ \sum_{t=1}^T \mathbb{E}_{z'_t} \left[ \ln \left( \frac{g(y_t|z'_t)h(z'_t|x_t)}{\hat{g}_t(y_t|z'_t)\hat{h}_t(z'_t|x_t)} \right) \Big| x_t, y_t \right] \right] 
    = \mathbb{E} \left[ \sum_{t=1}^T \mathbb{E}_{z_t} \left[ \ln \left( \frac{g(y_t|z_t)h(z_t|x_t)}{\hat{g}_t(y_t|z_t)\hat{h}_t(z_t|x_t)} \right) \Big| x_t, y_t \right] \right] %&& \textrm{Since } z_t \textrm{ is independent of } x^{t-1} \textrm{ and } y^{t-1}. 
    \\
    &= \mathbb{E} \left[ \sum_{t=1}^T \ln \left( \frac{g(y_t|z_t)h(z_t|x_t)}{\hat{g}_t(y_t|z_t)\hat{h}_t(z_t|x_t)} \right)  \right] %&& \textrm{By the tower rule.} \\
    = \mathbb{E} \left[ \sum_{t=1}^T \ln \left( \frac{g(y_t|z_t)}{\hat{g}_t(y_t|z_t)} \right)  \right] + \mathbb{E} \left[ \sum_{t=1}^T \ln \left( \frac{h(z_t|x_t)}{\hat{h}_t(z_t|x_t)} \right)  \right]~.
\end{align*}
%where in the third equality we used the tower rule.
\end{proof}

Now we are ready to prove Theorem~\ref{thm:dff_regret}.

\begin{proof}[Proof of Theorem~\ref{thm:dff_regret}] 
To prove the theorem, we bound the two terms on the right-hand side of \eqref{eqn:decomp} in Lemma~\ref{lem:factored_regret} separately. Applying  Theorem~\ref{thm:df_regret_bound}
(with $D \leftarrow D$, $\X \leftarrow \Z$, and $\Y \leftarrow \Y$), the first term can be bounded by 
\[
O\left( (D+1)|\Y| |\Z| \ln \left( \frac{T}{|\Z|}\right) \right)~.
\]
Writing the second term as
\[
\sum_{t=1}^T \ln \left( \frac{h(z_t|x_t)}{\hat{h}_t(z_t|x_t)} \right) = 
\sum_{x \in \X} \sum_{t=1}^T \mathbb{I}\{x_t=x\} 
\ln \left(\frac{h(z_t|x)}{\hat{h}_t(z_t|x)}\right),
\]
the inner sum can be bounded by the regret of the prediction algorithm $\hat{h} ^{(x)}$ applied to a sequence of length
$T_x = \sum_{t=1}^{T} \mathbb{I}\{x_t = x \}$, which is $O(|\Z| \ln T_x)$ by Remark~9.3 of \citet{Cesa2006}.
Thus,
\[
\sum_{t=1}^T \ln \left( \frac{h(z_t|x_t)}{\hat{h}_t(z_t|x_t)} \right)
= O\left(\sum_{x \in \X} |\Z|\ln T_x\right) 
\le O\left(|\Z| N \ln \left(\frac{T}{N}\right)\right)
\]
where the last inequality follows from Jensen's inequality since $\sum_{x\in\X} T_x=T$.
This completes the proof of the theorem.
\end{proof}

\section{Synthetic Prediction Task}

In this section we describe precisely how the synthetic prediction task used at the end of Section~\ref{sec:formal} is defined.
We generated a prediction task by sampling two stochastic matrices (with rows summing up to 1) -- an $|\X| \times |\Z|$ matrix $H$ (instances to proxies) and a $|\Z| \times |\Y|$ matrix $G$ (proxies to outcomes). To generate these matrices, we first created a matrix $R_\Z$ (and $R_\Y$) where each row was a one hot encoding of an element sampled uniformly from $\Z$ (or $\Y$). Next we generated a matrix $U_\Z$ (and $U_\Y$) with a uniform distribution in each row. The final matrix $P_\Z$ (and $P_\Y$) was created by interpolating $P_\Z = (1-\epsilon)R_\Z + \epsilon U_\Z$ (and $P_\Y = (1-\epsilon)R_\Y + \epsilon U_\Y$) for some $\epsilon \in [0, 1]$.

The instances were $\X = \{1, 2, \dots, N\}$. Let $\mu \in [0, 1]$, and $U_1, U_2, \dots, U_T$ be independent, uniform random variables over $\X$. The instances were selected according to
\begin{align*}
    x_t &= \left\{ \begin{array}{ll} U_t & \text{with probability $\mu$; and} \\ \min(N, \lfloor t / D \rfloor + 1) & \text{otherwise.} \end{array} \right.
\end{align*}
Once the instances were selected, the proxy was selected by sampling $z_t \in \Z$ according the distribution specified by the $x_t^{\rm th}$ row of the matrix $H$. Finally, the outcome was selected by sampling $y_t \in \Y$ according to the distribution specified by the $z_t^{\rm th}$ row of the matrix $G$.

When $\mu = 0$, the schedule is adversarial with a pattern that looks like:
\begin{align*}
    x_1 = 1, x_2 = 1, \dots, x_D = 1, x_{D+1} = 2, x_{D+2} = 2, \dots, x_{2D} = 2, ... , x_{ND} = N, x_{ND+1} = N, \dots, x_{T} = N \enspace .
\end{align*}
On the other hand, when $\mu = 1$, the instances are sampled uniformly from $\X$ at each round.

In our experiments we used the following parameter choices: the number of rounds $T = 1000$, the delay $D = 100$, the number of instances $N = 10$, the number of proxies $|\Z| = 4$, and the number of outcomes $|\Y| = 5$.

\end{document}

%% file: figures_models_input_to_label.tex
\begin{tikzpicture}

\node (input) [inp] {$\mathbf{x}$};
\path (input.north)+(0, 1) node (net) [mlp] {MLP};
\path (net.north)+(0, 1) node (logitz) [inp] {$\texttt{logits}(\hat{y})$};
\path (logitz.north)+(0, 1) node (softmax) [mlp] {Softmax};
\path (softmax.north)+(0, 1) node (predz) [inp] {$\hat{y}$};
\path (logitz.west)+(-2, 0) node (loss) [mlp] {Loss};
\path (loss.south)+(0, -2.4) node (z) [inp] {$y$};

\path [draw, ->] (input.north) -- node [above] {} (net.south);
\path [draw, ->] (net.north) -- node [above] {} (logitz.south);
\path [draw, ->] (logitz.west) -- node [above] {} (loss.east);
\path [draw, ->] (logitz.north) -- node [above] {} (softmax.south);
\path [draw, ->] (softmax.north) -- node [above] {} (predz.south);
\path [draw, ->] (z.north) -- node [above] {} (loss.south);

\end{tikzpicture}

%% file: figures_models_input_to_feedback.tex
\begin{tikzpicture}

\node (input) [inp] {$\mathbf{x}$};
\path (input.north)+(0, 1) node (net) [mlp] {MLP};
\path (net.north)+(0, 1) node (logitz) [inp] {$\texttt{logits}(\hat{z})$};
\path (logitz.north)+(0, 1) node (softmax) [mlp] {Softmax};
\path (softmax.north)+(0, 1) node (predz) [inp] {$\hat{z}$};
\path (logitz.west)+(-2, 0) node (loss) [mlp] {Loss};
\path (loss.south)+(0, -2.4) node (z) [inp] {$z$};

\path [draw, ->] (input.north) -- node [above] {} (net.south);
\path [draw, ->] (net.north) -- node [above] {} (logitz.south);
\path [draw, ->] (logitz.west) -- node [above] {} (loss.east);
\path [draw, ->] (logitz.north) -- node [above] {} (softmax.south);
\path [draw, ->] (softmax.north) -- node [above] {} (predz.south);
\path [draw, ->] (z.north) -- node [above] {} (loss.south);

\end{tikzpicture}

%% file: figures_models_feedback_to_label.tex
\begin{tikzpicture}

\node (input) [inp] {$z$};
\path (input.north)+(0, 1) node (net) [mlp] {MLP};
\path (net.north)+(0, 1) node (logitz) [inp] {$\texttt{logits}(\hat{y})$};
\path (logitz.north)+(0, 1) node (softmax) [mlp] {Softmax};
\path (softmax.north)+(0, 1) node (predz) [inp] {$\hat{y}$};
\path (logitz.west)+(-2, 0) node (loss) [mlp] {Loss};
\path (loss.south)+(0, -2.4) node (z) [inp] {$y$};

\path [draw, ->] (input.north) -- node [above] {} (net.south);
\path [draw, ->] (net.north) -- node [above] {} (logitz.south);
\path [draw, ->] (logitz.west) -- node [above] {} (loss.east);
\path [draw, ->] (logitz.north) -- node [above] {} (softmax.south);
\path [draw, ->] (softmax.north) -- node [above] {} (predz.south);
\path [draw, ->] (z.north) -- node [above] {} (loss.south);

\end{tikzpicture}

%% file: figures_models_residual.tex
\begin{tikzpicture}

\node (input) [inp] {$(\mathbf{x}, z)$};
\path (input)+(-1.5, 0) node (z) [inp] {$z$};
\path (z.north)+(0, 1) node (feedbacknet) [mlp] {MLP};
\path (feedbacknet.north)+(0, 1) node (logity) [inp] {
  \begin{tabular}{c}
    $\texttt{logits}(\hat{y})$\\
    (from feedback only)
  \end{tabular}
};

\path (input.north)+(2.5, 2.375) node (residual) [mlp] {MLP};
\path (residual.north)+(0, 1) node (corry) [inp] {$\Delta\texttt{logits}(\hat{y})$};
%\path (residual.west)+(-.3, 0) node (dummy) [null] {};
\path (logity.west)+(-1.5, 0) node (loss) [mlp] {Loss};
\path (loss.south)+(0, -2.5) node (y) [inp] {$y$};
\path (logity.south)+(0, 3) node (add) [op] {$+$};
\path (add.north)+(0, 1) node (logityfull) [inp] {$\texttt{logits}(\hat{y})$};
\path (logityfull.east)+(2, 0) node (softmax) [mlp] {Softmax};
\path (softmax.east)+(1, 0) node (predy) [inp] {$\hat{y}$};

\path [draw, ->] (input.west) --  (z.east);
\path [draw, ->] (z.north) -- node [above] {} (feedbacknet.south);
\path [draw, ->] (feedbacknet.north) -- node [above] {} (logity.south);
\path [draw, ->] (logity.west) -- node [above] {} (loss.east);
\path [draw, ->] (input.east) -| node [above] {} (residual.south);
\path [draw, ->] (residual.north) -- node [above] {} (corry.south);
\path [draw, ->] (y.north) -- node [above] {} (loss.south);
\path [draw, ->, dashed] (logity.north) -- node [left] {(no backprop)} (add.south);
\path [draw, ->] (corry.north) |- node [above] {} (add.east);

\path [draw, ->] (add.north) -- node [above] {} (logityfull.south);
\path [draw, ->] (logityfull.east) --   (softmax.west);
\path [draw, ->] (softmax.east) --   (predy.west);
\path [draw, ->] (logityfull.west) -| node [above] {} (loss.north);
\path [draw, ->] (logity.east) --  (residual.west);
%\path [draw, -] (logity.east) -| node [above] {} (dummy.north);
%\path [draw, ->] (dummy.east) -- node [above] {} (residual.west);

\end{tikzpicture}